\documentclass{article}
\usepackage[margin=1in]{geometry}
\usepackage{amsmath,amssymb,amsfonts,graphicx,amsthm,nicefrac,mathtools,bm}
\usepackage{color}
\usepackage[numbers]{natbib}
\usepackage[utf8]{inputenc}
\usepackage{hyperref}
\usepackage[algo2e]{algorithm2e}
\usepackage{algorithm}
\usepackage{bbm}

\newtheorem{corollary}{Corollary}
\newtheorem{theorem}{Theorem}
\newtheorem{proposition}{Proposition}

\theoremstyle{definition}
\newtheorem{definition}{Definition}
\theoremstyle{remark}
\newtheorem{remark}{Remark}

\theoremstyle{lemma}
\newtheorem{lemma}{Lemma}

\DeclareMathOperator*{\argmin}{arg\,min}

\newcommand{\A}{\mathcal{A}}

\newcommand{\D}{\mathcal{D}}

\newcommand{\N}{\mathbb{N}}

\newcommand{\C}{\mathcal{C}}
\newcommand{\X}{\mathcal{X}}
\newcommand{\Y}{\mathcal{Y}}
\newcommand{\E}{\mathbb{E}}

\renewcommand{\O}{\mathcal{O}}
\newcommand{\PP}[1]{\mathbb{P}\!\left\{{#1}\right\}}

\newcommand{\one}{\mathbf{1}}

\title{Private Prediction Sets}
\author{Anastasios N. Angelopoulos\thanks{equal contribution},~ Stephen Bates\footnotemark[1],~ Tijana Zrnic\footnotemark[1],~ Michael I.~Jordan\\ \\ University of California, Berkeley}

\begin{document}

\maketitle

\begin{abstract}
    In real-world settings involving consequential decision-making, the deployment of machine learning systems generally requires both reliable uncertainty quantification and protection of individuals' privacy.  We present a framework that treats these two desiderata jointly.  Our framework is based on conformal prediction, a methodology that augments predictive models to return \emph{prediction sets} that provide uncertainty quantification---they provably cover the true response with a user-specified probability, such as 90\%.  One might hope that when used with privately-trained models, conformal prediction would yield privacy guarantees for the resulting prediction sets; unfortunately this is \emph{not} the case.  To remedy this key problem, we develop a method that takes any pre-trained predictive model and outputs differentially private prediction sets. Our method follows the general approach of split conformal prediction; we use holdout data to calibrate the size of the prediction sets but preserve privacy by using a privatized quantile subroutine. 
This subroutine compensates for the noise introduced to preserve privacy in order to guarantee correct coverage.
We evaluate the method on large-scale computer vision datasets.
\end{abstract}

\section{Introduction}
The impressive predictive accuracies of black-box machine learning algorithms on tightly-controlled test beds do not sanctify their use in consequential applications.
For example, given the gravity of medical decision-making, automated diagnostic predictions must come with rigorous instance-wise uncertainty to avoid silent, high-consequence failures.
Furthermore, medical data science requires privacy guarantees, since individuals would suffer material harm were their data to be accessed or reconstructed by a nefarious actor.  While uncertainty quantification and privacy are generally dealt with in isolation, they arise together in many real-world predictive systems, and, as we discuss, they interact. Accordingly, the work that we present here involves a framework that addresses uncertainty and privacy jointly.
Specifically, we develop a differentially private version of conformal prediction that results in private, rigorous, finite-sample uncertainty quantification for any model and any dataset at little computational cost.

Our approach builds on the notion of \emph{prediction sets}---subsets of the response space that provably cover the true response variable with pre-specified probability (e.g., $90\%$). 
Formally, for a test point with feature vector $X \in \X$ and response $Y \in \Y$, we compute an uncertainty set function, $\C(\cdot)$, mapping a feature vector to a subset of $\Y$ such that
\begin{equation}
\label{eq:predictive_set_coverage}
\PP{Y \in \C(X)} \ge 1 - \alpha,
\end{equation}
for a user-specified confidence level $1-\alpha\in(0,1)$.
We use the output of an underlying predictive model (e.g., a pre-trained, privatized neural network) along with a held-out \emph{calibration dataset}, $\{(X_i,Y_i)\}_{i=1}^n$, from the same distribution as $(X,Y)$ to fit the set-valued function $\C(\cdot)$. 
The probability in expression~\eqref{eq:predictive_set_coverage} is therefore taken over both the randomness in $(X,Y)$ and $\{(X_i,Y_i)\}_{i=1}^n$.
If the underlying model expresses uncertainty, $\C$ will be large, signaling skepticism regarding the model's prediction.

Moreover, we introduce a \emph{differentially private} mechanism for fitting $\C$, such that the sets that we compute have low sensitivity to the removal of any calibration point. This will allow an individual to contribute a calibration data point without fear that the prediction sets will reveal their sensitive information.
Note that \emph{even if the underlying model is trained in a privacy-preserving fashion, this provides no privacy guarantee for the calibration data}. 
Therefore, we will provide an adjustment that masks the calibration dataset with additional randomness, addressing both privacy and uncertainty simultaneously.

\begin{figure*}[t]
    \centering
    \includegraphics[width=\textwidth]{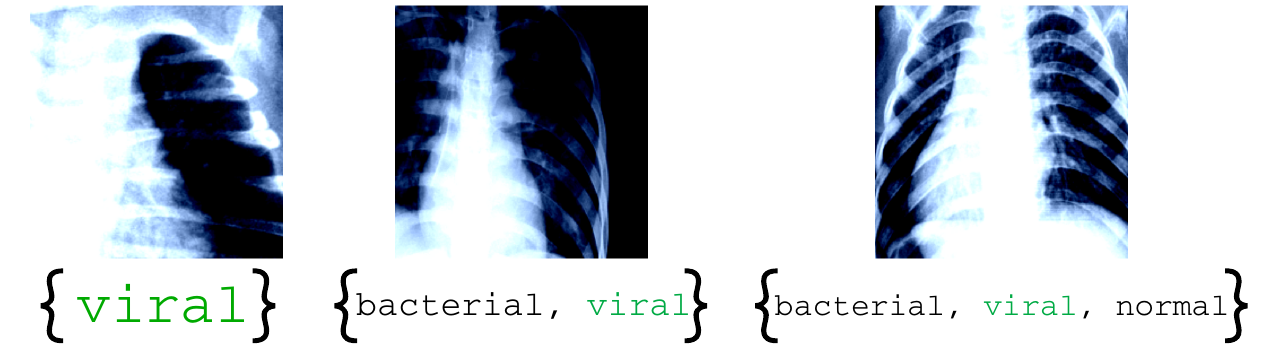}
    \vspace{-0.5cm}
    \caption{\textbf{Examples of private conformal prediction sets on COVID-19 data.}
    We show three examples of lung X-rays taken from the CoronaHack dataset~\cite{perez2020databiology} with their corresponding private prediction sets at $\alpha=10\%$ from a ResNet-18.
    All three patients had \texttt{viral pneumonia} (likely COVID-19).
    The classes in the prediction sets appear in ranked order according to the softmax score of the model; the center and right images are incorrectly classified if the predictor returns only the most likely class, but are correctly covered by the private prediction sets.
    See Experiment~\ref{sec:experiment4} for details.
    }
    \label{fig:three_covid}
\end{figure*}

See Figure~\ref{fig:three_covid} for a concrete example of private prediction sets applied to the automated diagnosis of COVID-19.
In this setting, the prediction sets represent a set of plausible diagnoses based on an X-ray image---either \texttt{viral pneumonia} (presumed COVID-19), \texttt{bacterial pneumonia}, or \texttt{normal}.
We guarantee that the true diagnosis is contained in the prediction set with high probability, while simultaneously ensuring that an adversary 
cannot detect the presence of any one of the X-ray images used to train the predictive system.

\subsection{Our contribution}

Our main contribution is a privacy-preserving algorithm which takes as input any predictive model together with a calibration dataset, and outputs a set-valued function $\C(\cdot)$ that maps any input feature vector $X$ to a set of labels such that the true label $Y$ is contained in the predicted set with probability at least $1-\alpha$, as per equation~\eqref{eq:predictive_set_coverage}. In order to generate prediction sets satisfying this property, we will use ideas from  split conformal prediction~\citep{papadopoulos2002inductive, vovk2005algorithmic, lei2018distribution}, modifying this approach to ensure privacy.
Importantly, if the provided predictive model is also trained in a differentially private way, then the whole pipeline that maps data to a prediction set function $\C(\cdot)$ is differentially private as well.

In Algorithm~\ref{alg:main_algo}, we sketch our main procedure.

\begin{algorithm}[H]
\SetAlgoLined
\textbf{input:} predictor $\hat{f}(\cdot)$, calibration data $\{(X_i,Y_i)\}_{i=1}^n$, privacy level $\epsilon>0$, confidence level $\alpha\in(0,1)$\newline
For $1\leq i\leq n$, compute conformity score $s_i = S_{\hat f}(X_i,Y_i)$\newline
Compute $\epsilon$-differentially private $(1-\alpha + O((n\epsilon)^{-1}))$-quantile of $\{s_i\}_{i=1}^n$, denoted $\hat s$\newline
\textbf{output:} $\C(\cdot) = \{y:S_{\hat f}(\cdot,y)\leq \hat s\}$
\caption{Private prediction sets (informal)}
\label{alg:main_algo}
\end{algorithm}

Algorithm~\ref{alg:main_algo} first computes the conformity scores for all training samples. Informally, these scores indicate how well a feature--label pair ``conforms'' to the provided model $\hat{f}$, a low score implying high conformity and a high score being indicative of an atypical point from the perspective of $\hat{f}$. Then, the algorithm generates a certain carefully chosen private quantile of the scores. Finally, it returns a prediction set function $\C(\cdot)$ which, for a given input feature vector, returns all labels that result in a conformity score below the critical threshold $\hat s$.

Our main theoretical result asserts that Algorithm~\ref{alg:main_algo} has strict coverage guarantees and is differentially private. In addition, we show that the coverage is almost \emph{tight}, that is, not much higher than $1-\alpha$.

\begin{theorem}[Informal preview]
\label{thm:informal}
	The prediction set function $\C(\cdot)$ returned by Algorithm~\ref{alg:main_algo} is $\epsilon$-differentially private and satisfies
	$$1 - \alpha \leq \PP{Y\in \C(X)} \leq 1-\alpha + O((n\epsilon)^{-1}).$$
\end{theorem}

We obtain a gap between the lower and upper bound on the probability of coverage to be roughly of the order $O((n\epsilon)^{-1})$, similar to the standard gap $O(n^{-1})$ without the privacy requirement. With this, we provide the first theoretical insight into the cost of privacy in conformal prediction. To shed further light on the properties of our procedure, we perform an extensive empirical study where we evaluate the tradeoff between the level of privacy on one hand, and the coverage and size of prediction sets on the other.

\subsection{Related work}
Differential privacy \cite{dwork2006calibrating} has become the de facto standard for privacy-preserving data analysis, as witnessed by its widespread adoption in large-scale systems such as those by Google \cite{erlingsson2014rappor, bittau2017prochlo}, Apple \cite{apple_privacy}, Microsoft \cite{ding2017collecting}, and the US Census Bureau \cite{abowd2018us, dwork2019differential}.
This increasing adoption of differential privacy goes hand in hand with steady progress in differentially private model training, ranging across both convex \cite{chaudhuri2011differentially, bassily2014private} and non-convex \cite{abadi2016deep, neel2020oracle} settings. 
Our work complements these works by proposing a procedure that can be combined with any differentially private model training algorithm to account for the uncertainty of the resulting predictive model by producing a prediction set function with formal guarantees. 
At a technical level, closest to our algorithm on the privacy side are existing methods for reporting histograms and quantiles in a privacy-preserving fashion \cite{dwork2006calibrating, xu2013differentially, lei2011differentially, smith2011privacy, feldman2017generalization}.
Finally, there have also been significant efforts to quantify uncertainty with formal privacy guarantees through various types of private confidence intervals \cite{karwa2017finite, sheffet2017differentially, gaboardi2019locally, wang2019differentially}. 
While prediction sets resemble confidence intervals, they are fundamentally different objects as they do not aim to cover a fixed parameter of the population distribution, but rather a randomly sampled outcome. 
As a result, existing methods for differentially private confidence intervals do not generalize to our problem setting.

Prediction sets as a way to represent uncertainty are a classical idea, going back at least to tolerance regions in the 1940s \cite{wilks1941, wilks1942, wald1943, tukey1947}. See Krishnamoorthy \& Mathew~\cite{krishnamoorthy2009statistical} for an overview of tolerance regions and Park et al.~\cite{Park2020PAC} for a recent application to deep learning models.
Conformal prediction \cite{vovk1999machine, vovk2005algorithmic, shafer2008tutorial} is a related way of producing predictive sets with finite-sample guarantees. Most relevant to the present work, \emph{split conformal prediction} \cite{papadopoulos2002inductive, lei2013conformal, lei2018distribution} is a convenient version that uses data splitting to give prediction sets in a computationally efficient way. Vovk~\cite{vovk2015cross} and Barber et al.~\cite{barber2019jackknife} refine this approach to re-use data for both training and calibration, improving statistical efficiency.
Recent work has targeted desiderata such as small set sizes \cite{Sadinle2016LeastAS, angelopoulos2020sets}, coverage that is approximately balanced across feature space \cite{vovk2012conditional, barber2019limits, romano2019conformalized, izbicki2019flexible, romano2020classification, guan2020conformal, cauchois2020knowing}, and coverage that is balanced across classes \cite{lei2014classification, Sadinle2016LeastAS, hechtlinger2018cautious, guan2019prediction}. Further extensions address problems in distribution estimation \cite{vovk2017nonparametric, vovk2019conformal}, handling or testing distribution shift \cite{tibshirani2019conformal, cauchois2020robust, hu2020distributionfree}, causal inference \cite{lei2020conformal}, and controlling other notions of statistical error \cite{bates2021distributionfree}. 
We suggest~\cite{angelopoulos2021gentle} and \cite{shafer2008tutorial} as introductory tutorials on conformal prediction for the unfamiliar reader.
Lastly, we highlight two alternative approaches with a similar goal to conformal prediction. First, the calibration technique in Jung et al.~\cite{jung2020moment} and Gupta et al.~\cite{gupta2021online} generates prediction sets via the estimation of higher moments across many overlapping sub-populations. Second, there is a family of techniques that define a utility function balancing set-size and coverage and then search for set-valued predictors to maximize this utility \cite{grycko1993classification, delcoz2009learning, mortier2020efficient}. The present work builds on split conformal prediction, but modifies the calibration step to preserve privacy.

\section{Preliminaries}

In this section, we formally introduce the main concepts in our problem setting. 
Split conformal prediction assumes access to a predictive model, $\hat{f}$, and aims to output \emph{prediction sets} that achieve coverage by quantifying the uncertainty of $\hat{f}$ and the intrinsic randomness in $X$ and $Y$.
It quantifies this uncertainty using a \emph{calibration dataset} consisting of $n$ i.i.d.\ samples, $\{(X_i,Y_i)\}_{i=1}^n$, that were not used to train $\hat{f}$. 
The calibration proceeds by defining a \emph{score function} $S_{\hat{f}}: \mathcal{X} \times \mathcal{Y} \to \mathbb{R}$.  Without loss of generality we take the range of this function to be the unit interval $[0, 1]$.
The reader should think of the score as measuring the degree of consistency of the response $Y$ with the features $X$ based on the predictive model $\hat{f}$ (e.g., the size of the residual in a regression model), but any score function would lead to correct coverage. To simplify notation we will write $S(\cdot,\cdot)$ to denote the score, where we implicitly assume an underlying model $\hat{f}$. 
From this score function, one forms prediction sets as follows:
\begin{equation}
\label{eq:predictive_set_definition}
    \C(x) = \{y : S(x, y) \le \hat{s}\},
\end{equation}
for a choice of $\hat{s}$ based on the calibration dataset.
In particular, $\hat{s}$ is taken to be a quantile of the calibration scores $s_i = S(X_i, Y_i)$ for $i=1,\dots,n$.
In non-private conformal prediction, one simply takes $\hat{s}$ to be the $\big((n+1)(1-\alpha)\big) / n$ quantile, and then a standard argument shows that the coverage property in \eqref{eq:predictive_set_coverage} holds. In this work we show how to take a modified private quantile that maintains this coverage guarantee.

As a concrete example of standard split conformal prediction, consider classifying an image in $\mathcal{X}=\mathbb{R}^{m\times d}$ into one of a thousand classes, $\mathcal{Y}=\{1,...,1000\}$.
Given a standard classifier outputting a probability distribution over the classes, $\hat{f}:\mathcal{X} \to [0,1]^{1000}$ (e.g., the output of a softmax layer), we can define a natural score function based on the activation of the correct class, $S(x,y)=1-\hat f(x)_y$.
Then we take $\hat{s}$ as the upper $\lceil 0.9 (n+1)\rceil / n$ quantile of the calibration scores $s_1,\dots, s_n$ and define $\C$ as in equation~\eqref{eq:predictive_set_definition}. That is, we take as the cutoff $\hat{s}$ the value such that if we include all classes with estimated probability greater than $1 - \hat{s}$, our sets have (only slightly more than) 90\% coverage on the calibration data.
The result $\C(x)$ on a test point is then a set of plausible classes guaranteed to contain the true class with probability 90\%. Our proposed method will follow a similar workflow, but with a slightly different choice of $\hat{s}$ to guarantee both coverage and privacy.

We next formally define differential privacy. We say that two datasets $\D,\D'$ are \emph{neighboring} if they differ in a single element, i.e., either dataset can be obtained from the other by removing a single entry. For example, $\D\in(\X\times\Y)^n$ and $\D' = \D\setminus \{(X_0,Y_0)\}$, for some $(X_0,Y_0)\in \D$. Differential privacy then requires that two neighboring datasets produce similar distributions on the output.
\begin{definition}[Differential privacy \cite{dwork2006calibrating}]
A randomized algorithm $\A$ is $\epsilon$-\emph{differentially private} if for all neighboring datasets $\D$ and $\D'$, it holds that:
$$\PP{\A(\D) \in \O} \leq e^\epsilon \PP{\A(\D') \in \O},$$
for all measurable sets $\O$.
\end{definition}
In short, if no adversary observing the algorithm's output can distinguish between $\D$ and a dataset $\D'$ with the $i$-th entry removed, the presence of individual $i$ in the analysis cannot be detected and hence their privacy is not compromised.

A key ingredient to our procedure is a privatized quantile of the conformity scores. We obtain this private quantile by discretizing the scores into bins and applying the exponential mechanism \cite{mcsherry2007mechanism}, one of the most ubiquitous tools in differential privacy. Our private quantile routine is then an extension of the private median routine proposed by Feldman and Steinke~\citep{feldman2017generalization} to handle arbitrary quantiles. Specifically, let us fix a number of bins $m\in\N$, as well as edges $0 \equiv e_0 < e_1 < ... < e_{m-1} < e_m \equiv 1$. The edges define the bins $I_j = (e_{j-1}, e_j]$, $j=1, ..., m$.
We use Algorithm~\ref{alg:private_quantile} with appropriately chosen quantile level $q$ as a subroutine of our main conformal procedure.

\begin{algorithm}[H]
\SetAlgoLined
{\bf input:} calibration scores $\{s_1,\dots,s_{n}\}$, bins $\{I_1,\dots, I_m\}$, privacy level $\epsilon$, level $q\in[0,1]$\newline
For all $1\leq i\leq n$, compute discretized score $[s_i] = e_j$, where $s_i \in I_j$\newline
For all $1\leq j\leq m$, compute $w_j = \max\left\{\frac{\#\{i: [s_i] < e_j\}}{q},\frac{\#\{i: [s_i] > e_j\}}{1-q}\right\}$\vspace{0.1cm}\newline
Let $\hat s = e_j$ with probability $e^{-\frac{\epsilon \min(q, 1-q) w_j}{2}}/\sum_{j'=1}^m e^{-\frac{\epsilon \min(q, 1-q) w_{j'}}{2}}$\newline
{\bf output:} private quantile $\hat s$
\caption{Differentially private quantile}
\label{alg:private_quantile}
\end{algorithm}



\section{Main algorithm and guarantees}

We next precisely state our main algorithm and its formal guarantees. First, our algorithm has a calibration step, Algorithm~\ref{alg:private_conformal}, carried out one time using the calibration scores $s_1,\dots,s_n$ as input; this is the heart of our proposed procedure. The output of this step is a cutoff $\hat{s}$ learned from the calibration data. With this in hand, one forms the prediction set for a test point $x$ as in equation~\eqref{eq:predictive_set_definition}, which for completeness we state in Algorithm~\ref{alg:private_prediction_set}.

\begin{algorithm}[H]
\SetAlgoLined
\textbf{input:} calibration scores $\{s_1,\dots,s_{n}\}$, privacy parameter $\epsilon$, coverage level $\alpha$, bins $\{I_1,\dots,I_m\}$\newline
\hspace{0.08cm} Compute $\tilde q$-quantile of $\{s_1,\dots,s_n\}$ via Algorithm~\ref{alg:private_quantile}, where $\tilde q$ is defined in \eqref{eq:q-level}, denoted $\hat s$\newline
{\bf output:} calibrated score cutoff $\hat{s}$
\caption{Differentially private calibration}
\label{alg:private_conformal}
\end{algorithm}

\begin{algorithm}[H]
\SetAlgoLined
\textbf{input:} test point $x$, calibrated score cutoff $\hat{s}$\newline
{\bf output:} prediction set  as in \eqref{eq:predictive_set_definition}: $\C(x) = \{y : S(x,y) \le \hat{s} \}$.
\caption{Differentially private prediction set}
\label{alg:private_prediction_set}
\end{algorithm}

This algorithm both satisfies differential privacy and guarantees correct coverage, as stated next in Proposition~\ref{prop:dp_conformal_is_private} and Theorem~\ref{thm:coverage}, respectively. 
The privacy property is a straightforward consequence of the privacy guarantees on the exponential mechanism~\cite{mcsherry2007mechanism}.

\begin{proposition}[Privacy guarantee]
\label{prop:dp_conformal_is_private}
Algorithm~\ref{alg:private_conformal} is $\epsilon$-differentially private.
\end{proposition}

Therefore, the main challenge for theory lies in understanding how to compensate for the added differentially private noise in order to get strict, distribution-free coverage guarantees.

\begin{theorem}[Coverage guarantee]
\label{thm:coverage}
Fix the differential privacy level $\epsilon>0$ and miscoverage level $\alpha\in (0.5,1)$, as well as a free parameter $\gamma\in(0,1)$. 
Let
\begin{equation}
\label{eq:q-level}
\tilde q = \frac{(n+1)(1-\alpha)}{n(1-\gamma\alpha)} + \frac{2}{\epsilon n}\log\left(\frac{m}{\gamma\alpha}\right),
\end{equation}
and let $\hat s$ be the output of Algorithm \ref{alg:private_quantile} at level $\min\{\tilde q,1\}$. Then, the prediction sets in \eqref{eq:predictive_set_definition} with cutoff $\hat s$ satisfy the coverage property in~\eqref{eq:predictive_set_coverage}.
\end{theorem}

\begin{remark}
We can choose $\gamma$ to minimize $\tilde q$, which leads to smallest prediction sets. 
The optimal value $\gamma^*$ depends only on $n,m,$ and $\alpha$, and can be found by taking a derivative of~\eqref{alg:private_quantile}; see Appendix~\ref{app:exp_details}.
\end{remark}

Note that the significance level $\tilde{q}$ in~\eqref{eq:q-level} is just a slightly inflated version of the non-private conformal quantile: $\tilde{q} \ge \frac{(n+1)(1-\alpha)}{n} \ge 1 - \alpha$. 
Indeed, taking $\epsilon \to \infty, \gamma \to 0$ in~\eqref{alg:private_quantile} recovers the non-private quantile.
Intuitively, we must raise the significance level to compensate for the noise introduced to preserve privacy. 
We informally sketch the main ideas in the proof, deferring the details to the Appendix.

\begin{proof}[Proof sketch]
We can write the probability of coverage as:
\begin{equation*}
    \PP{Y \in \C(X)} = \E\left[F(\hat s)\right],
\end{equation*}
where $F$ is the distribution of appropriately discretized empirical scores. We observe that for all $\tilde q$, the exponential mechanism with input $\tilde q$ and $s_1,\dots,s_n$ returns an empirical quantile no smaller than the $\tilde q - O(1/(n\epsilon))$ empirical quantile. This allows us to write
\begin{equation*}
\E\left[F(\hat s)\right] \geq (1-\gamma\alpha) \E\left[F(\hat {F}^{-1}(\tilde q - O(1/(n\epsilon)))\right],
\end{equation*}
where $\hat F$ denotes the empirical distribution of the discretized scores. For any $q$, the random variable $F(\hat {F}^{-1}(q))$ is distributed as the $\lceil n q\rceil$-th order statistic of a super-uniform distribution, which implies that it can be stochastically lower bounded by the $\lceil n q\rceil$-th order statistic of a uniform distribution. This order statistic follows a beta distribution with known parameters, whose expectation can hence be evaluated analytically. Carefully choosing $\tilde q$ as a function of this expectation completes the proof of the theorem.
\end{proof}

\begin{figure}[t]
    \centering
    \includegraphics[width=\textwidth]{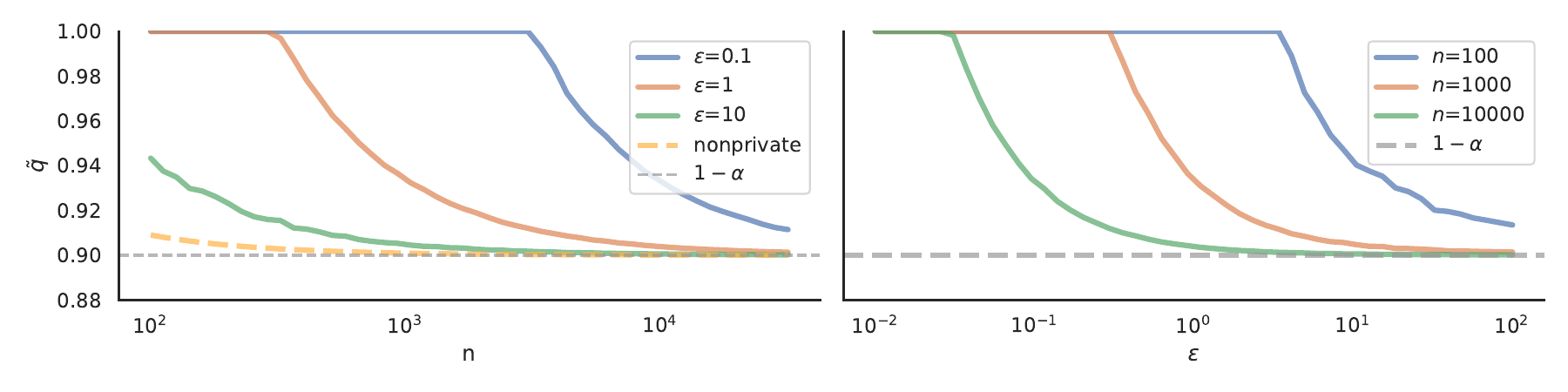}
    \vspace{-0.7cm}
    \caption{\textbf{The private quantile $\tilde{q}$ as $n$ and $\epsilon$ grow.}
    We demonstrate the adjusted quantile from~\eqref{eq:q-level} as $n$ and $\epsilon$ increase, with automatically chosen values for $m$ and $\gamma$ described in Appendix~\ref{app:exp_details}.
    As the number of samples grows and the privacy constraint relaxes, the procedure chooses a less conservative quantile, eventually approaching the limiting value $1-\alpha$.
    The mild fluctuations in the curves are due to differing choices of $m^*$ and $\gamma$.
    }
    \label{fig:qhat}
\end{figure}

With the validity of Algorithm \ref{alg:private_conformal} established, we next prove that the algorithm is not too conservative in the sense that the coverage is not far above $1-\alpha$.
A key quantity in our upper bound is 
$$p_{\max}^m := \max_{1\leq j\leq m} \PP{s_1 \in I_j}.$$
This quantity captures the impact of the score discretization. Smaller $p_{\max}^m$ corresponds to mass spread more evenly throughout the bins.
For well-behaved score functions, we expect $p_{\max}^m$ to scale as $O(m^{-1})$. Indeed, if the scores have any continuous density on $[0,1]$ bounded above and we take uniformly spaced bins, then $p_{\max}^m = O(m^{-1})$. In terms of $p_{\max}^m$, we have the following upper bound.
\begin{theorem}[Coverage upper bound]
\label{thm:coverage_upper}
The prediction sets in \eqref{eq:predictive_set_definition} with $\hat s$ is as in Theorem \ref{thm:coverage}, satisfy the following coverage upper bound:
\begin{equation*}
    \PP{Y \in \C(X)} \leq 1 - (1 - \gamma)\alpha + (1-\gamma\alpha)\left( 2p_{\textnormal{max}}^m + \frac{2\left(1 +\max\left\{\frac{\tilde q}{1-\tilde q},1\right\}\right)\log(m/(\gamma\alpha))}{(n+1) \epsilon }\right),
\end{equation*}
where $\tilde q$ is defined in \eqref{eq:q-level}.
\end{theorem}

If we further assume a weak regularity condition on the scores, then by balancing the rates in the expression above we arrive at an explicit upper bound.
\begin{corollary}[Coverage upper bound, simplified form]
\label{corollary:coverage_upper_simplified}
Suppose that the input scores follow a continuous distribution on $[0,1]$ with a density that is bounded above. 
Take $m \propto n\epsilon$ and $\gamma = 1/m$. Then, the prediction sets in \eqref{eq:predictive_set_definition}, with $\hat s$ as in Theorem \ref{thm:coverage}, satisfy the following upper bound:
\begin{equation*}
    \PP{Y \in \C(X)} \leq 1 - \alpha +  O\left(\frac{\log\big(n\epsilon/\alpha\big)}{n\epsilon}\right).
\end{equation*}
\end{corollary}


We emphasize that the assumptions on the score distribution are only needed to prove the upper bound; the coverage lower bound holds for any distribution. In any case, these assumptions are very weak, essentially requiring only that the score distribution contains no point masses. In fact, this requirement could even be enforced ex post facto by adding a small amount of tiebreaking noise, in which case we would need no restrictions on the input distribution of scores whatsoever.

The upper bound answers an important practical question: how many bins should we take? If $m$ is too small, then there is little noise addition due to privacy, but the histogram only coarseis an overly coarse approximatesion of the empirical distribution of the scores. 
On the other hand, if $m$ is too large, then the histogram is accurate, but the private quantile in~\ref{eq:q-level} can grow as wellre is a lot of additive noise implied by the requirement of differential privacy. This tension can be observed in the termss in Theorem~\ref{thm:coverage_upper} that haveve a dependence on $m$. Corollary~\ref{corollary:coverage_upper_simplified} suggests that the correct balance---which leads to minimal excess coverage---is to take $m \propto n\epsilon$.
In practice, because the dependence of $\tilde{q}$ on $m$ is only logarithmic, $m$ is often very large.

This upper bound also gives insight to an important theoretical question: what is the cost of privacy in conformal prediction? In non-private conformal prediction, the upper bound is $1-\alpha + O(n^{-1})$ \cite{lei2018distribution}. In private conformal prediction, we achieve an upper bound of $1-\alpha + \tilde O((n\epsilon)^{-1})$, a relatively modest cost incurred by privacy-preserving calibration.

\section{Experiments}
We now turn to an empirical evaluation of differentially private conformal prediction for image classification problems.
In this setting, each image $X_i$ has a single unique class label $Y_i \in \{1,...,K\}$ estimated by a predictive model $\hat{f} : \mathcal{X} \to [0,1]^K$.
We seek to create private prediction sets, $\C(X_i) \subseteq \{1,...,K\}$, achieving coverage as in equation~\eqref{eq:predictive_set_coverage}, using the following score function:
\begin{equation*}
    S(x,y) = 1 - \hat{f}(x)_y,
\end{equation*}
as in Sadinle et al.~\cite{Sadinle2016LeastAS}.
This section evaluates the prediction sets generated by Algorithm~\ref{alg:private_conformal} by quantifying the cost of privacy and the effects of the model, number of calibration points, and number of bins used in our procedure.
We use the CIFAR-10 dataset~\cite{krizhevsky2009learning} wherever we require a privately trained neural network.
Otherwise, we use a non-private model on the ImageNet dataset~\cite{deng2009imagenet}, to investigate the performance of our procedure in a more challenging setting with a large number of possible labels. 
Except where otherwise mentioned, we use an automated number of uniformly spaced bins $m^*$ to construct the privatized CDF.
Appendix~\ref{app:exp_details} describes the algorithm for choosing an approximately optimal value of $m^*$ when the conformal scores are roughly uniform based on fixed values of $n$, $\epsilon$, and $\alpha$.
We finish the section by providing private prediction sets for diagnosing viral pneumonia on the CoronaHack dataset~\cite{perez2020databiology}. 
The reader can reproduce the experiments exactly using our \href{https://github.com/aangelopoulos/private_prediction_sets}{\textcolor{blue}{public GitHub repository}}.
\begin{figure}[t]
    \centering
    \vspace{-0cm}
    \includegraphics[width=0.45\textwidth]{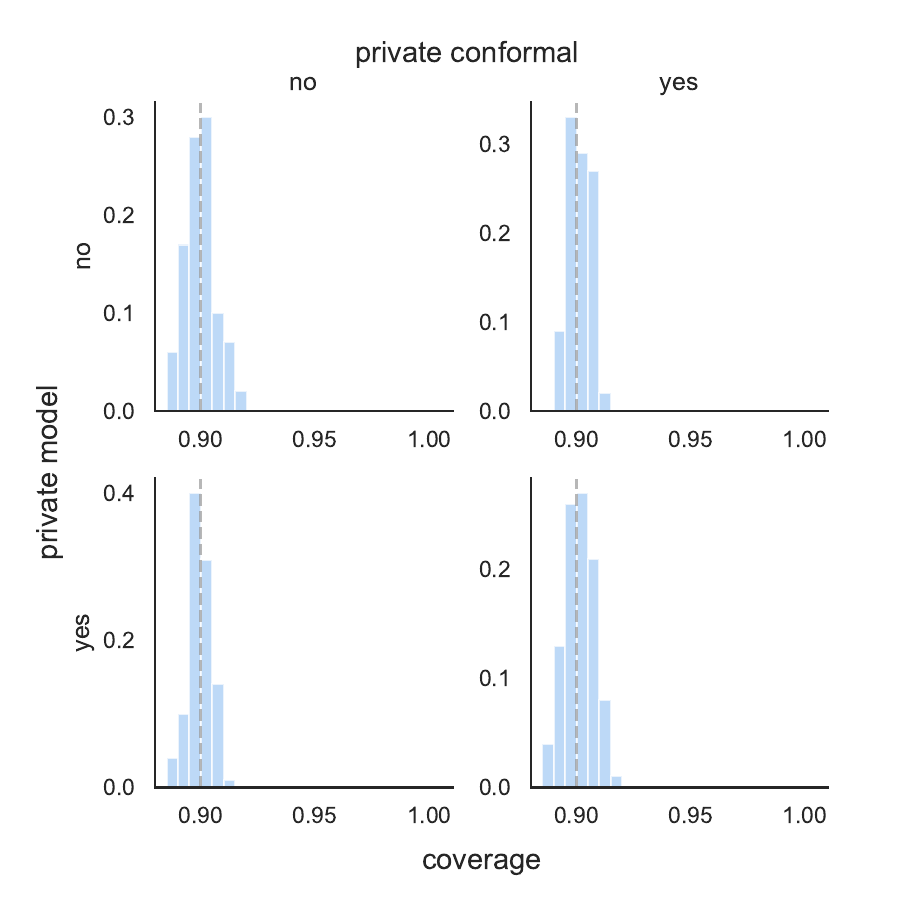}
    \includegraphics[width=0.45\textwidth]{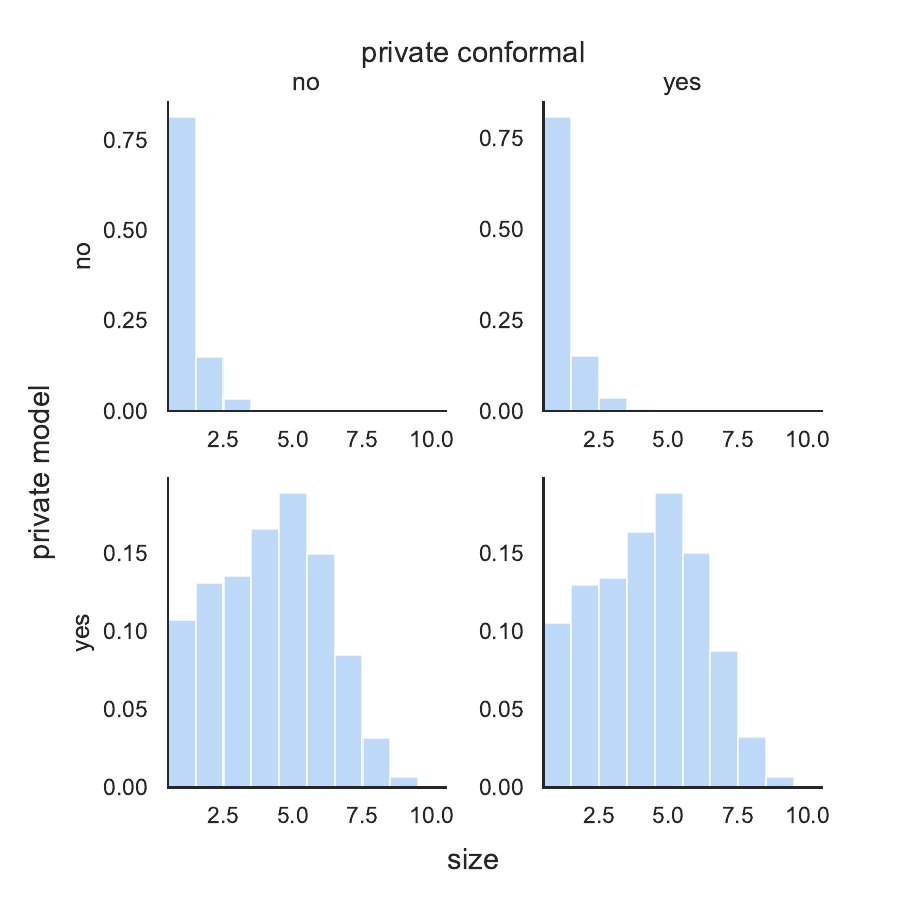}
    \vspace{-0.6cm}
    \caption{\textbf{Coverage and set size with private/non-private models and private/non-private conformal prediction.} 
    We demonstrate histograms of coverage and set size of non-private/private models and non-private/private conformal prediction at the level $\alpha=0.1$, with $\epsilon=8$, $\delta=1e-5$, and $n=5000$.}
    \label{fig:experiment1}
    \vspace{-0cm}
\end{figure}

\begin{figure}
    \centering
    \includegraphics[width=\textwidth]{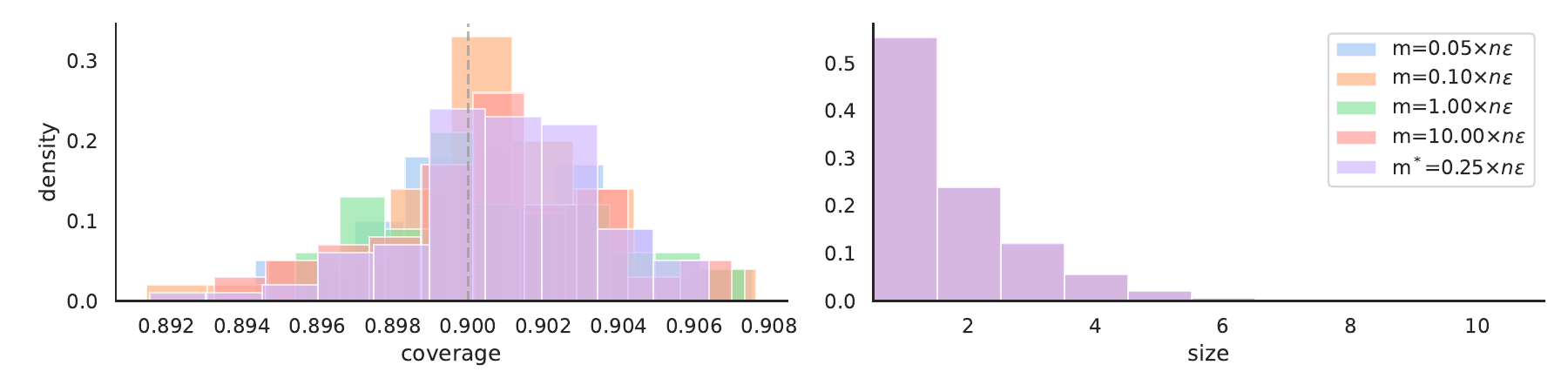}
    \vspace{-0.7cm}
    \caption{\textbf{Coverage and set size for different values of $m$.} 
    We demonstrate the performance on Imagenet of private conformal prediction using a non-private ResNet-152 as the base model at $\alpha=0.1$ and $\epsilon=5$. The coverage is nearly constant over three orders of magnitude of bin numbers.
    All of the histograms on the right hand side are overlapping.
    See Section~\ref{sec:experiment2} for details.}
    \label{fig:experiment2}
\end{figure}

\begin{figure}[t]
    \centering
    \includegraphics[width=\textwidth]{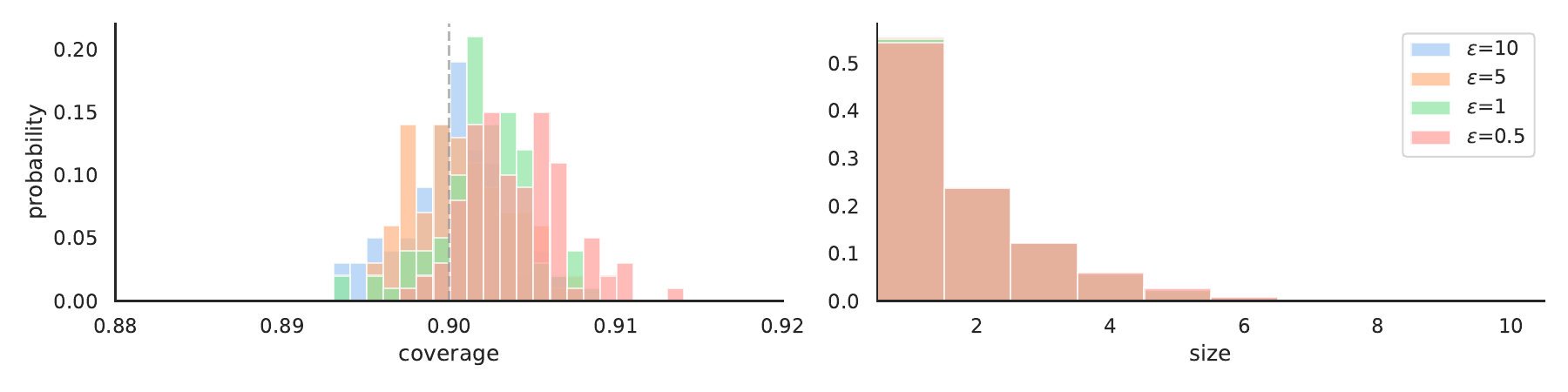}
    \vspace{-0.7cm}
    \caption{\textbf{Coverage and set size for different values of $\epsilon$.} 
    We demonstrate the performance on ImageNet of private conformal prediction using a non-private ResNet-152 as the base model with $\alpha=0.1$. The coverage improves slightly for liberal (large) $\epsilon$, although the cost of privacy is evidently very low. 
    All of the histograms on the right hand side are overlapping.
    See Section~\ref{sec:experiment3} for details.}
    \label{fig:experiment3}
\end{figure}
\begin{figure}[t]
    \centering
    \includegraphics[width=\textwidth]{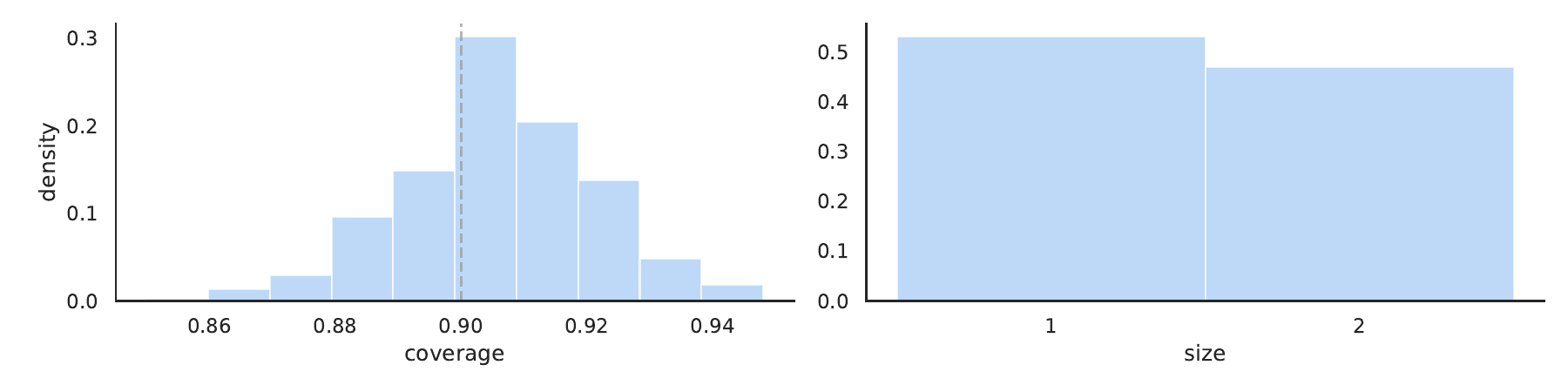}
    \vspace{-0.7cm}
    \caption{\textbf{Coverage and set size on the CoronaHack dataset.} 
    We demonstrate the performance on the CoronaHack dataset of private conformal prediction using a non-private ResNet-18 as the base model with $\alpha=0.1$ and $\epsilon = 10$. 
    The median coverage was $90.4\%$.
    See Section~\ref{sec:experiment4} for details.}
    \label{fig:experiment4}
\end{figure}

\subsection{Isolating the effects of private model training and private conformal prediction}
\label{sec:experiment1}

We would like to disentangle the effects of private conformal prediction from those of private model training.
To that end, we report the coverage and set sizes of the following four procedures: private conformal prediction with a private model, non-private conformal prediction with a private model, private conformal prediction with a non-private model, and non-private conformal prediction with a non-private model.
The non-private model and private model are both the same stock convolutional architecture from the \texttt{Opacus} library.
The private model is trained with private SGD \cite{abadi2016deep}, as implemented in the \texttt{Opacus} library, with privacy parameters $\epsilon=8$ and $\delta=1e-5$.
We used the suggested private model training parameters from the \texttt{Opacus} library (see Appendix~\ref{app:exp_details}), as our work does not aim to improve private model training.
The non-private model's accuracy ($73\%$) was significantly higher than that of the private model ($67\%$).

Figure~\ref{fig:experiment1} shows histograms of the coverages and set sizes of these procedures over 1000 random splits of the CIFAR-10 validation set with $n=5000$.
Notably, the results show the price of private conformal prediction is very low, as evidenced by the minuscule increase in set size caused by private conformal prediction.
However, the private model training causes a larger set size due to the private model's comparatively poor performance.
Note that a user desiring a fully private pipeline will use the procedure in the bottom right quadrant of the plot.

\subsection{Varying number of bins $m$}
\label{sec:experiment2}

Here we probe the performance of private prediction sets as the number of uniformly spaced bins $m$ in our procedure changes.
Based on our theoretical results, $m$ should be on the order of $n\epsilon$, with the exact number dependent on the underlying model and the choices of $\alpha$, $n$, and $\epsilon$.
A too-small choice of $m$ coarsely quantizes the scores, so Algorithm~\ref{alg:private_prediction_set} may be forced to round up to a very conservative private quantile. 
A too-large choice of $m$ increases the logarithmic term in~\eqref{eq:q-level}.
The optimal choice of $m$ balances these two factors.

To demonstrate this tradeoff, we performed experiments on ImageNet.
We used a non-private, pre-trained ResNet-152 from the \texttt{torchvision} repository as the base model.
Figure~\ref{fig:experiment2} shows the coverage and set size of private prediction sets over 100 random splits of ImageNet's validation set for several choices of $m$; we used $n=30000$ and evaluated on the remaining $20000$ images.
The experimental results suggest $m^*$ works comparatively well, and that our method is relatively insensitive to the number of bins over several orders of magnitude.

\subsection{Varying privacy level $\epsilon$}
\label{sec:experiment3}

Next we quantify how the coverage changes with the privacy parameter $\epsilon$.
We used $n=30000$ calibration points and $20000$ evaluation points as in Experiment~\ref{sec:experiment3}.
For each value of $\epsilon$ we choose a different value of $m^*$.
Figure~\ref{fig:experiment3} shows the coverage and set size of private prediction sets over 100 splits of ImageNet's validation set for several choices of $\epsilon$.
As $\epsilon$ grows, the procedure becomes less conservative.
Overall the procedure exhibits little sensitivity to $\epsilon$.

\subsection{COVID-19 diagnosis}
\label{sec:experiment4}

Next we show results on the CoronaHack dataset, a public chest X-ray dataset containing $5908$ X-rays labeled as \texttt{normal}, \texttt{viral pneumonia} (primarily COVID-19), or \texttt{bacterial pneumonia}.
Using $4408$ training pairs over $14$ epochs, we (non-privately) fine-tuned the last layer of a pretrained ResNet-18 from \texttt{torchvision} to predict one of the three diagnoses.
The private conformal calibration procedure saw a further $n=1000$ examples, and we used the remaining $500$ for validation.
The ResNet-18 had a final accuracy of $75\%$ after fine-tuning.
Figure~\ref{fig:experiment4} plots the coverage and set size of this procedure over $1000$ different train/calibration/validation splits of the dataset, and Figure~\ref{fig:three_covid} shows selected examples of these sets.

\section{Discussion}
We introduce a method to produce differentially private prediction sets that contain the true response with a user-specified probability by blending split conformal prediction with differentially private quantile computation. 
The primary challenge we resolve in this work is simultaneously satisfying the coverage property and privacy property, which requires a careful choice of the conformal score threshold to account for the added privacy noise. 
Our corresponding upper bound shows that the coverage does not greatly exceed the nominal level $1-\alpha$, meaning that our procedure is not too conservative. 
Moreover, our upper bound gives insight into the price of privacy in conformal prediction: the upper bound scales as $\tilde O((n \epsilon)^{-1})$ compared to $O(n^{-1})$ for non-private conformal prediction, a mild decrease in efficiency. 
This is confirmed in our experiments, where we show that there is little difference between private and non-private conformal prediction when using the same predictive model. 
We also observe the familiar phenomenon that there is a substantial decrease in accuracy for private model fitting compared to non-private model fitting. 
We conclude that the cost of privacy lies primarily in the model fitting---private calibration has a comparatively minor effect on performance. We also note that any improvement in private model training would immediately translate to smaller prediction sets returned by our method.
In sum, we view private conformal prediction as an appealing method for uncertainty quantification with differentially private models.

\section*{Acknowledgements}
We would like to thank Pierre Humbert and Batiste Le Bars for pointing out a typo in the original statement of Algorithm \ref{alg:private_quantile}.

\bibliographystyle{ieeetr}
\bibliography{conformal, rcps}

\appendix

\section{Auxiliary results}

We start with a result about the error of the private quantile mechanism, stated in Algorithm \ref{alg:private_quantile}. The following is an extension of the the analogous result for the private median due to Feldman and Steinke~\cite{feldman2017generalization}.
\begin{lemma}
\label{lemma:DP_quantile_accuracy}
For any $\delta\in(0,1)$, the differentially private quantile algorithm (Algorithm \ref{alg:private_quantile}) satisfies:
\begin{equation*}
    \PP{ \frac{1}{n}\#\{i:[s_i]\leq \hat s\} \geq q -  \frac{2\max\left\{\frac{1-q}{q},1\right\}\log(m/\delta)}{n\epsilon} }\geq 1-\delta
\end{equation*}
and
\begin{equation*}
    \PP{ \frac{1}{n}\#\{i:[s_i]< \hat s\} \leq q +  \frac{2\max\left\{\frac{q}{1-q},1\right\}\log(m/\delta)}{n\epsilon} }\geq 1-\delta.
\end{equation*}
\end{lemma}

\begin{proof}
By the standard utility guarantee for the exponential mechanism \cite{mcsherry2007mechanism} (e.g., Corollary 3.12 in \cite{dwork2014algorithmic}), we have
\begin{equation}
\label{eq:exp_mech_utility}
\PP{\max\left\{\frac{\#\{j:[s_j] < \hat s\}}{q},\frac{\#\{j:[s_j] > \hat s\}}{1-q}\right\} < \min_i w_i + \frac{2\max\{1/q,1/(1-q)\}\log(m/\delta)}{\epsilon}} \geq 1-\delta.
\end{equation}
First we argue that $\min_i w_i \leq n$. Let $s^* = \min\{s\in\{e_0,\dots,e_m\}: \#\{i: [s_i] \leq s\}>qn\}$. Then, $\#\{i: [s_i] > s^*\} < (1-q)n$ trivially. Furthermore, $\#\{i: [s_i] < s^*\} \leq q n$ by the definition of $s^*$, since $s^*$ is the \emph{first} point at which the cumulative fraction of scores less than or equal to $s^*$ exceeds $q$. Since we have identified a bin where $\max\left\{\frac{\#\{j:[s_j] <  s^*\}}{q},\frac{\#\{j:[s_j] > s^*\}}{1-q}\right\} \leq n$, we can conclude that $\min_i w_i \leq n$.

Going back to equation~\eqref{eq:exp_mech_utility}, we have that with probability at least $1-\delta$,
\begin{align*}
\frac{1}{n}\#\{i:s_i\leq \hat s\} &= 1 - \frac{1}{n}\#\{i: s_i > s\}\\
&\geq 1 - (1-q) \frac{\min_i w_i}{n} - \frac{2\max\{(1-q)/q,1\}\log(m/\delta)}{n\epsilon}\\
&\geq q - \frac{2\max\{(1-q)/q,1\}\log(m/\delta)}{n\epsilon}.
\end{align*}
Similarly,
\begin{align*}
\frac{1}{n}\#\{i:s_i< \hat s\}
&\leq q\frac{\min_i w_i}{n} + \frac{2\max\{1,q/(1-q)\}\log(m/\delta)}{n\epsilon}\\
&\leq q + \frac{2\max\{1,q/(1-q)\}\log(m/\delta)}{n\epsilon}.
\end{align*}
\end{proof}

Next, we package some classical facts about the distribution of order statistics in a form helpful for analyzing conformal prediction.
\begin{lemma}
\label{lemma:beta_lemma}
Let $F$ be the CDF of a distribution supported on a finite set $\{a_1,\dots,a_m\}$. Let $Z_1,\dots,Z_n\stackrel{i.i.d.}{\sim}~F$, and let $\hat F$ denote the empirical CDF corresponding to $Z_1,\dots,Z_n$. Denote also $p_{\max}^m = \max_{1\leq i\leq m}\PP{Z_1 = a_i}$. Then, 
$$Z_{\textsc{Beta}} + p_{\max}^m \succeq F(\hat{F}^{-1}(q)) \succeq Z_{\textsc{Beta}},$$
where $Z_{\textsc{Beta}}$ follows the beta distribution $\textsc{Beta}(\lceil nq\rceil, n - \lceil nq\rceil + 1)$ and $\succeq$ denotes first-order stochastic dominance.
\end{lemma}

\begin{proof}
Since we take $\hat F^{-1}(q) = \inf\{z: \hat {F}(z) \ge q\}$ by definition, then that implies $\hat F^{-1}(q) = Z_{(\lceil nq\rceil)}$, where $Z_{(i)}$ denotes the $i$-th non-decreasing order statistic of $Z_1,\dots,Z_n$. By monotonicity of $F$, we further have that $F(Z_{(\lceil nq\rceil)})$ is identical to the $\lceil nq\rceil$-th non-decreasing order statistic of $F(Z_1),\dots,F(Z_n)$. By a standard argument, the samples $F(Z_1),\dots,F(Z_n)$ are super-uniform, i.e. $\PP{F(Z_1)\leq u}\leq u$ for all $u\in[0,1]$. In other words, they are stochastically larger than a uniform distribution on $[0,1]$, and thus their $\lceil nq\rceil$-th order statistic is stochastically lower bounded by the $\lceil nq\rceil$-th order statistic of a uniform distribution, which follows the $\textsc{Beta}(\lceil n\alpha\rceil, n - \lceil n\alpha\rceil + 1)$ distribution. This completes the proof of the lower bound. For the upper bound, we use the fact that $\PP{F(Z_1) \leq u} \geq u - p_{\max}^m$, and so $F(Z_i)$ are stochastically dominated by $U_i + p_{\max}^m$, where $\{U_i\}_{i=1}^n$ are i.i.d. uniform on $[0,1]$. Their $\lceil nq\rceil$-th order statistic is distributed as $Z_{\textsc{Beta}} + p_{\max}^m$, which completes the proof.
\end{proof}


\section{Proofs}

\subsection{Proof of Theorem \ref{thm:coverage}}

First we introduce some notation. By $F$ we will denote the discretized CDF of the scores; in particular, for any $i\in\{1,\dots,n\}$,
$$F(s) = \PP{ [s_i] \leq s}.$$
Here, by $[s_i]$ we denote a \emph{discretized} version of $s_i$ where we set $[s_i] = e_j$ if $s_i \in I_j$. We also let $\hat F$ denote the empirical distribution of the discretized scores:
$$\hat F(s) =\frac{1}{n}\sum_{i=1}^n \one\{[s_i]\leq s\}.$$

By convention, we let $F^{-1}(\delta)$ denote the left-continuous inverse of $F$, i.e. $F^{-1}(\delta) := \inf \{s: F(s)\geq \delta\}$, and we similarly define $\hat F^{-1}(\delta)$.

We can write
\begin{equation*}
    \PP{Y \in \C(X)} = \PP{S(X,Y) \leq \hat s} = \E\left[F(\hat s)\right].
\end{equation*}

Denote the event $E=\left\{ \frac{1}{n}\#\{i:[s_i] \leq \hat s\} \geq \tilde q - \frac{2}{\epsilon n}\log(m/(\gamma\alpha))\right\}$, and note that by Lemma \ref{lemma:DP_quantile_accuracy} and the fact that $\tilde q \geq 0.5$, $\PP{E}\geq 1-\gamma\alpha$. By splitting up the analysis depending on $E$, we obtain the following:
\begin{align*}
    \E\left[F(\hat s)\right] &= \E\left[F(\hat s)\one\{E^c\}\right] + \E\left[F(\hat s)\one\{E\}\right]\\
    &\geq \gamma\alpha\cdot 0  + \E\left[F(\hat s)\one\{E\}\right]\\
    &\geq (1-\gamma\alpha) \E\left[F\left(\hat{F}^{-1}\left(\tilde q - \frac{2}{\epsilon n}\log\left(m/(\gamma\alpha)\right)\right)\right) \right],
\end{align*}
where the final inequality follows by the definition of $E$. Thus, it suffices to show that 
\begin{equation}
\label{eq:coverage_lb}
    \E\left[F\left(\hat{F}^{-1}\left(\tilde q - \frac{2}{\epsilon n}\log\left(m/(\gamma\alpha)\right) \right)\right) \right] \ge \frac{1-\alpha}{1-\gamma\alpha}.
\end{equation}
Let $j^* = \left\lceil n \left(\tilde q - \frac{2}{\epsilon n}\log\left(m/(\gamma\alpha)\right) \right)\right\rceil$. Then, by Lemma \ref{lemma:beta_lemma},
\begin{equation*}
    F\left(\hat{F}^{-1}\left(\tilde q - \frac{2}{\epsilon n}\log\left(m/(\gamma\alpha)\right)\right)\right) \succeq \textsc{Beta}(j^*, n - j^* + 1),
\end{equation*}
so
\begin{equation*}
    \E\left[F\left(\hat{F}^{-1}\left(\tilde q - \frac{2}{\epsilon n}\log\left(m/(\gamma\alpha)\right)\right)\right)\right] \ge \frac{j^*}{ n+1} = \frac{\lceil n (\tilde q - \frac{2}{\epsilon n}\log\left(m/(\gamma\alpha)\right) )\rceil}{n+1}.
\end{equation*}
By the definition of $\tilde q$, we see that
\begin{equation*}
    \frac{\lceil n (\tilde q - \frac{2}{\epsilon n}\log\left(m/(\gamma\alpha)\right) )\rceil}{n+1} \ge \frac{1-\alpha}{1-\gamma\alpha},
\end{equation*}
holds, which implies equation~\eqref{eq:coverage_lb} and thus completes the proof.

\subsection{Proof of Theorem \ref{thm:coverage_upper}}

We adopt the definitions of $F$, $\hat F$ from Theorem \ref{thm:coverage}, and define $E$ as the
event
\begin{equation*}
\left\{ \frac{1}{n}\#\{i:[s_i]< \hat s\} \leq \tilde q +  \frac{2\max\{\frac{\tilde q}{1-\tilde q},1\}\log(m/\gamma\alpha)}{n\epsilon} \right\},
\end{equation*}
which by Lemma~\ref{lemma:DP_quantile_accuracy} 
has probability at least $1-\gamma\alpha$.
By a similar reasoning as in Theorem \ref{thm:coverage}, we obtain the following:
\begin{align}
\label{eq:upper_bound1}
    \E\left[F(\hat s)\right] &= \E\left[F(\hat s)\one\{E^c\}\right] + \E\left[F(\hat s)\one\{E\}\right] \nonumber\\
    &\leq \gamma\alpha\cdot 1  + \E\left[F(\hat s)\one\{E\}\right] \nonumber\\
    &\leq \gamma\alpha + (1-\gamma\alpha) \left(\E\left[F\left(\hat{F}^{-1}\left(\tilde q + \frac{2\max\{\frac{\tilde q}{1-\tilde q},1\}}{\epsilon n}\log\left(m/(\gamma\alpha)\right) \right)\right) \right] + p^m_{\max}\right),
\end{align}
where the final inequality follows by the definition of $E$.

Let $j^* = \lceil n \left(\tilde q + \frac{2\max\{\frac{\tilde q}{1-\tilde q},1\}}{\epsilon n}\log\left(m/(\gamma\alpha)\right) \right)\rceil$. By Lemma \ref{lemma:beta_lemma}, we have
\begin{equation*}
    F\left(\hat{F}^{-1}\left(\tilde q + \frac{2\max\{\frac{\tilde q}{1-\tilde q},1\}}{\epsilon n}\log\left(m/(\gamma\alpha)\right) \right)\right) \preceq \textsc{Beta}(j^*, n - j^* + 1) + p_{\max}^m,
\end{equation*}
so
\begin{equation}
\label{eq:upper_bound2}
    \E\left[F\left(\hat{F}^{-1}\left(\tilde q + \frac{2\max\{\frac{\tilde q}{1-\tilde q},1\}}{\epsilon n}\log\left(m/(\gamma\alpha)\right)\right)\right) \right] \le \frac{j^*}{ n+1} + p_{\max}^m = \frac{\lceil n \left(\tilde q + \frac{2\max\{\frac{\tilde q}{1-\tilde q},1\}}{\epsilon n}\log\left(m/(\gamma\alpha)\right) \right)\rceil}{n+1} + p_{\max}^m.
\end{equation}
By the definition of $\tilde q$, we see that
\begin{align}
\label{eq:upper_bound3}
    \frac{\lceil n \left(\tilde q + \frac{2\max\{\frac{\tilde q}{1-\tilde q},1\}}{\epsilon n}\log\left(m/(\gamma\alpha)\right) \right)\rceil}{n+1} &\le \frac{\frac{1-\alpha}{1-\gamma\alpha}(n+1) + \frac{2(1 +\max\{\frac{\tilde q}{1-\tilde q},1\})}{\epsilon}\log\left(m/(\gamma\alpha)\right) }{n+1} \nonumber\\
    &= \frac{1-\alpha}{1-\gamma\alpha} + \frac{2(1 +\max\{\frac{\tilde q}{1-\tilde q},1\})\log(m/(\gamma\alpha))}{(n+1) \epsilon }.
\end{align}
Putting together equations \eqref{eq:upper_bound1}, \eqref{eq:upper_bound2}, and \eqref{eq:upper_bound3} completes the proof.

\section{Experimental details}
\label{app:exp_details}

\paragraph{Choosing $m^*$ and $\gamma$.}
Algorithm~\ref{alg:mstar-gamma} gives automatic choices of the optimal number of uniformly spaced bins, $m^*$, and the tuning parameter $\gamma$ that work well for approximately uniformly distributed scores.
In a moment, we will show how to find the optimal value $\gamma^*$ for a fixed value of $m$. 
Once $\gamma^*$ gets chosen, we will simulate uniformly distributed scores to choose the value $m^*$ that results in the best quantile for specific, pre-determined values of $\alpha$, $\epsilon$, and $n$. 
In practice, $m^*$ can be chosen from a relatively coarse grid of, say, 50 values logarithmically spaced from $10^2$ to $10^6$.

We start choosing the optimal value $\gamma^*$ by solving for the zeros of the derivative $\frac{\delta\tilde{q}}{\delta\gamma}$, leading to the quadratic expression,
\begin{equation}
    \label{eq:optimal_gamma_derivative}
    \frac{\delta\tilde{q}}{\delta \gamma} = 0 \Longleftrightarrow \alpha^2\gamma^2 -\left( \frac{\alpha(1-\alpha)\epsilon(n+1)}{2} +2\alpha \right)\gamma + 1 = 0.
\end{equation}
Letting $\Gamma$ be the roots of~\eqref{eq:optimal_gamma_derivative}, we can then choose the optimal value $\gamma^*$ as
\begin{equation}
\label{eq:optimal_gamma}
    \gamma^* = \argmin_{\gamma \in \Gamma \cap (0,1) \cup \{1e-12\}} \left[ \frac{(n+1)(1-\alpha)}{n(1-\gamma\alpha)} + \frac{2}{\epsilon n}\log\left(\frac{m}{\gamma\alpha}\right)\right],
\end{equation}
where the number 1e-12 takes care of the case that both roots lie outside the interval $(0,1)$.

\begin{algorithm}[H]
\SetAlgoLined
\SetKwInOut{Input}{input}
\Input{number of calibration points $n$, privacy level $\epsilon>0$, confidence level $\alpha\in(0,1)$}
\hspace{0.08cm} Simulate $n$ uniform conformity scores $s_i \sim {\rm Unif}(0,1), i=1,...,n$\newline
Choose $m^*$ to be the value of $m$ minimizing the output of Algorithm 3 on the $s_i$ with the optimal $\gamma^*$ chosen by~\eqref{eq:optimal_gamma}. \newline
\textbf{output:} $m^*$, $\gamma^*$
\caption{Get optimal number of bins and $\gamma$}
\label{alg:mstar-gamma}
\end{algorithm}

\paragraph{Private training procedure.}
We used the \texttt{Opacus} library with the default parameter choices included in the \hyperlink{https://github.com/pytorch/opacus/blob/master/examples/cifar10.py}{CIFAR-10 example code}.
The only difference in the non-private model training is the use of the \texttt{--disable-dp} flag, turning off the added noise but preserving all other settings.

\end{document}